\documentclass{article}

\usepackage[preprint,nonatbib]{tackling_climate_workshop_style}
\makeatletter
\renewcommand{\@noticestring}{Preprint.}
\makeatother

\usepackage[utf8]{inputenc}
\usepackage[T1]{fontenc}
\usepackage{url}
\usepackage{booktabs}
\usepackage{tabularx}
\usepackage{amsmath,amssymb,amsfonts,amsthm}
\usepackage{nicefrac}
\usepackage{microtype}
\usepackage{graphicx}
\usepackage{xspace}
\usepackage{natbib}
\usepackage{algorithm}
\usepackage[noend]{algpseudocode}
\usepackage{comment}
\excludecomment{reviewresponse}
\usepackage{dsfont}

\usepackage{xcolor}
\usepackage{hyperref}
\usepackage[nameinlink,noabbrev]{cleveref}
\newcommand{\ic}[1]{\textcolor{magenta}{\textbf{IC-} #1}}
\newcommand{\fc}[1]{\textcolor{blue}{\textbf{FC-} #1}}

\hypersetup{hidelinks}

\newtheorem{proposition}{Proposition}
\newtheorem{theorem}{Theorem}

\newcommand{\windowRho}{2.5\xspace}                  
\newcommand{\targetBudget}{96\xspace}                
\newcommand{\randomSeedCount}{30\xspace}             

\newcommand{\devRows}{33\xspace}                     
\newcommand{\devMDPs}{10\xspace}                     

\newcommand{\informativeRows}{17\xspace}             

\newcommand{\allWin}{25\xspace}
\newcommand{\allTie}{8\xspace}

\newcommand{\infWin}{14\xspace}
\newcommand{\infTie}{3\xspace}

\newcommand{\critVsBetter}{6/9/2\xspace}

\newcommand{\heldoutRows}{12\xspace}
\newcommand{\heldoutNonzero}{8\xspace}
\newcommand{\heldoutWindowedWin}{5\xspace}
\newcommand{\heldoutWindowedTie}{3\xspace}

\newcommand{\queueStates}{36\xspace}
\newcommand{\inventoryStates}{41\xspace}
\newcommand{\busStates}{90\xspace}

\newcommand{\windowEvaluationCap}{96\xspace}
\newcommand{\aedesSixtyOneEvals}{9\xspace}
\newcommand{\fireWindowCandidates}{58\xspace}
\newcommand{\fireZeroEvals}{3\xspace}
\newcommand{\forestWindowCandidates}{45\xspace}
\newcommand{\forestZeroEvals}{4\xspace}
\newcommand{\reserveWindowCandidates}{608\xspace}
\newcommand{\reserveZeroEvals}{15\xspace}
\newcommand{\fishThirteenCands}{499\xspace}

\newcommand{\aedesEvalCount}{\windowEvaluationCap}

\newcommand{\aedesSixtyOneBinary}{0.1981\xspace}
\newcommand{\gouldEightCandidates}{203\xspace}
\newcommand{\gouldEightBinary}{30.000\xspace}
\newcommand{\gouldEightWindowed}{9.344\xspace}
\newcommand{\fireBinary}{0.8227\xspace}
\newcommand{\forestBinary}{64.92\xspace}
\newcommand{\reserveBinary}{0.1484\xspace}
\newcommand{\fishThirteenBinary}{0.08264\xspace}
\newcommand{\fishThirteenWindowed}{0.002592\xspace}

\newcommand{\sonaFiftyBinary}{0.1441\xspace}
\newcommand{\sonaFiftyWindowed}{0.02331\xspace}

\newcommand{\sonaFiftyCandidates}{10{,}569\xspace} 

\newcommand{\gouldThirteenStates}{162\xspace}
\newcommand{\gouldThirteenK}{13\xspace}
\newcommand{\gouldThirteenCandidates}{374\xspace}
\newcommand{\gouldThirteenEvals}{39\xspace}
\newcommand{\gouldThirteenBinary}{14.380\xspace}
\newcommand{\gouldThirteenBinaryAgreement}{91.98\%\xspace}
\newcommand{\gouldThirteenBinaryDisagreements}{13\xspace}
\newcommand{\gouldThirteenWindowedAgreement}{100\%\xspace}

\newcommand{\gouldWitnessState}{38\xspace}
\newcommand{\gouldWitnessBinaryAction}{0 (DN)\xspace}
\newcommand{\gouldWitnessOptimalAction}{1 (FG)\xspace}
\newcommand{\aedesBinary}{0.03576\xspace}
\newcommand{\aedesWindowed}{0.000386\xspace}

\xspaceaddexceptions{\%,\,}
\DeclareRobustCommand{\MDP}{MDP\xspace}
\DeclareRobustCommand{\MDPs}{MDPs\xspace}

\DeclareRobustCommand{\MDPsFull}{Markov decision processes (MDPs)\xspace}
\DeclareRobustCommand{\MOMDP}{MOMDP\xspace}

\DeclareRobustCommand{\KMDP}{\texorpdfstring{$K$-MDP}{K-MDP}\xspace}

\DeclareRobustCommand{\AKMDP}{\texorpdfstring{A-$K$-MDP}{A-K-MDP}\xspace}

\DeclareRobustCommand{\WindowedAKMDP}{\texorpdfstring{Windowed A-$K$-MDP}{Windowed A-K-MDP}\xspace}
\DeclareRobustCommand{\SONA}{SONA\xspace}
\DeclareRobustCommand{\RepairedBinary}{endpoint-repaired Binary\xspace}

\title{\WindowedAKMDP}

\author{%
  Xiangwen Yang \\
  Monash University \\
  Melbourne, Australia \\
  \texttt{wayne.yang@monash.edu} \\
  \And
  Frankie Cho \\
  Monash University \\
  Melbourne, Australia \\
  \texttt{frankie.cho@monash.edu} \\
  \And
  Iadine Chades \\
  Monash University \\
  Melbourne, Australia \\
  \texttt{iadine.chades@monash.edu} \\
}

\begin{document}
\maketitle

\begin{abstract}
\MDPsFull are used to support decision-making in conservation of biodiversity, but policies, even over small state spaces, can
be difficult to interpret for conservation managers. \KMDP methods address this problem by building simpler MDPs with at
most $K$ abstract states. We show that the previously proposed \AKMDP algorithm that relies on selecting a discretisation divisor using binary search can skip better abstract states. To fix this issue, we propose \WindowedAKMDP, an algorithm that generates every distinct feasible partition induced within a declared divisor window and
evaluates candidates until reaching the ideal value loss ($J=0$) or exhausting the
family of candidates. Across 33 $K$-MDP instances, Windowed improved 25 and tied 8.

\end{abstract}

\section{Introduction}

\begin{reviewresponse}

\ic{MDPs in conservation - and K-MDPs - Frankie can have a go - keep it short}

\fc{Frankie had a go... will add a bit more lit review below. Note that we will need to fix the citation key bug... not sure if it's caused by the comment package. I wrote my comments without using AI. Fine to pass the text to the LLM for editorial suggestions but I would try to avoid directly using LLM outputs to replace entire sentences/ paragraphs as it may carry detectable AI watermarks.}

\fc{Main comments:}

\fc{- Are we presenting the algorithm to solve the Windowed K-MDP? It's in the initial draft but not sure where it is in in this draft}

\fc{- Note 4 page limit: I suggest we restrict the results to 1 MDP only to reduce the length of the results section and make our points clearer. We also need to introduce the case study setting - I'll work on it tomorrow.}

\fc{- We have to try to reduce the use of jargon such as "certificate", "semantics", "search diagnostic", "budget closure", "certificate-frontier" and "fixed-vocabulary". Are there simpler ways to say the same thing?
}
\end{reviewresponse}

\begin{reviewresponse}

\fc{Markov Decision Processes (MDP) are increasingly applied to advise policy-setting in climate, biodiversity and ecosystem protection settings. For example, MDPs were used to identify optimal strategies to protect critical ecosystems from decline, following a structured ``adaptive management" learning-by-doing approach \cite{chades2012momdps}. MDP solutions are also critical in advising policymakers in setting climate policies that determine optimal climate change mitigation policies and navigate uncertainties \citep{Traeger2014-lb}. Because many sustainability MDPs involve many states and actions, human decision-makers responsible for implementing these solutions need to distil complex optimal policies into clear, interpretable rules that can be carefully interrogated, socialized, trusted, and sometimes defended in a regulatory setting. Explainable AI (XAI) is therefore exceptionally vital in sustainability applications, where interpretation and trust in AI solutions are critical to success in practice.}

\fc{To make MDP solutions interpretable, researchers are increasingly turning to state abstraction algorithms that construct what are known as $K$-MDPs \cite{fm2020}. $K$-MDP algorithms abstract a complex solved MDPs with $N$ states into a smaller $K$ number of ``abstract" states \cite{fm2020,fm2024}, where $K$ is a compactness budget restricting the maximum number of abstract states in the simplified version of the MDP. $K$-MDPs provide a key XAI solution to sustainability challenges because they can help simplify complex sequential decision-making rules for practitioners. For example, \cite{Chades2012-zb} solved an 819-state, 4-action MDP problem to manage two competing threatened species, the sea otter and the Northern Abalone, resulting in an optimal policy that carefully accounted for conflicts and interactions among various ecological management priorities. The optimal policy consisted of optimal actions across all 819 states, making it difficult to explain to policymakers. $K$-MDPs provided a compelling solution to this challenge by making these MDP solutions far more human-readable, successfully representing the same MDP in only 10 states while only losing 2.8\% of its value compared to the value of the optimal policy.}

\fc{While $K$-MDP algorithms have been successful in practice, many limitations still hinder trust towards these methods. Many existing methods, such as the $ \phi_{a^*_d} $ and $ \phi_{Q_d} $ algorithms in Ferrer-Mestres et al. \cite{fm2020}, use binary search methods to intelligently search for an optimal discretization parameter $d$ to ensure that the number of states induced by the parameter $N(d)$ is less than the compactness budget $K$. Binary search algorithms, when applied in this context, make a prohibitive assumption that the state-budget feasibility predicate $\mathbb{1}\{N(d)\le K\}$ is monotone in $d$, which may not be true. This means that the binary search algorithm can potentially miss solutions that achieve the compactness budget $K$ while minimizing the value loss as part of the search process}

 \fc{In this paper we present a new solution to find near-optimal interpretable solutions for MDPs, known as Windowed A-$K$-MDPs, that overcome the assumption of monotonicity as part of the discretization process. }
\end{reviewresponse}

\MDPs have been used to inform sequential decisions under uncertainty in many conservation of biodiversity problems \citep{marescot2013complex}. However, MDP policies even for small states space can be difficult to interpret by humans, preventing opportunities to guide managers on the ground. $K$-MDPs algorithms have been proposed to address this problem by reducing the original MDP state space to $K$ abstract states \citep{fm2020,fm2024}. For example, the original sea-otter and northern-abalone model, with 819 states and four actions \citep{Chades2012-zb} was transformed into a ten-state \AKMDP with a small loss of performance (value loss = 2.8\%) \citep{fm2020}. 

In \AKMDP, the action/value abstraction $\phi_{a^*_d}$ groups states sharing an optimal
action and a discretised optimal value. Its published divisor-selection loop
uses midpoint binary updates \citep{fm2020}. At each iteration, the update direction is determined by whether the resulting abstraction contains at most $K$ states,
$N(d)\le K$. However, this condition is not monotone in the divisor $d$, so the binary search can
skip a feasible partition with lower value loss. Instead, our proposed \WindowedAKMDP
derives the exact values at which the induced partition changes. In its complete
branch, it generates each distinct feasible partition inside a fixed interval
around the repaired binary anchor and evaluates until reaching no value loss ($J=0$) or
exhausting the family. Its guarantee is local to that declared interval and
abstraction family.

Compared with the published binary-search procedure, we make three contributions. First, we provide a counterexample showing that the number of abstract states is not monotone in \(d\), and that binary search can therefore miss a feasible partition with lower decision loss. Second, within a specified  divisor window, we derive the exact values of \(d\) at which the induced partition can changes. Third, we compare \WindowedAKMDP with a repaired Binary baseline using the same construction and evaluation procedure. Across the conservation problems considered, \WindowedAKMDP finds compact policies with lower decision loss when better partitions are missed by binary search.

\section{Problem formulation}

Let $M=(\mathcal S,\mathcal A,P,r,\gamma)$ be a finite \MDP, where
$\mathcal S$ is the state space, $\mathcal A$ is the action set,
$\varnothing\ne\mathcal A(s)\subseteq\mathcal A$ is the set of actions
available at state $s$, $P$ is the transition kernel, $r$ is the reward, and
$\gamma\in[0,1)$ is the discount factor. Assume $V^*(s)\ge0$ and
$V_{\max}:=\max_s V^*(s)>0$. Lowest-index tie-breaking fixes an optimal
deterministic policy $\pi^*$. 

A \KMDP $M_K =(\mathcal S_K,\mathcal A, P_K, r_K, \gamma, \phi)$ is an MDP with at most $K$ states and solving a \KMDP problem means finding the best reduced state space ($|\mathcal S_K| \leq K$) so that the value loss between the ground MDP and the \KMDP is minimal. Here, we study the empirically best-performing \KMDP variant reported by \citet{fm2020} that uses the action/value abstraction, formally:
\begin{equation}
  \phi_{a^*_d}(s)=\bigl(\pi^*(s),\lceil V^*(s)/d\rceil\bigr),
  \label{eq:family}
\end{equation}
and write $\phi_d:=\phi_{a^*_d}$. States with the same action/value-bin pair
form one abstract state, so $N(d):=
  \left|\left\{\phi_{d}(s):s\in\mathcal S\right\}\right|$
is the induced number of abstract states. A divisor is feasible when
$N(d)\le K$. More generally, let
$N(\phi):=|\{\phi(s):s\in\mathcal S\}|$, so $N(\phi_d)=N(d)$. We restrict
attention to budgets satisfying
$N(V_{\max})\le K$, the minimum block count attainable by this divisor family
on $(0,V_{\max}]$. 




For a state mapping $\phi$, we define its abstract state set $\mathcal S_\phi:=\{\phi(s):s\in\mathcal S\}$
  and constituent blocks $B_k:=\{s\in\mathcal S:\phi(s)=k\}$.
Following \citet{abel2016near,fm2020}, we use uniform within-block weights
$\omega_\phi(s\mid k):=1/|B_k|$ for $s\in B_k$. The action set for abstract state $k$ is
$\mathcal A_\phi(k):=\bigcap_{s\in B_k}\mathcal A(s)$.
When $\phi=\phi_d$, $\mathcal A_\phi(k)\ne\varnothing$ for every
$k\in\mathcal S_\phi$, because $\pi^*(s)$ is the same for all $s\in B_k$ and
this common action belongs to $\mathcal A(s)$ for every $s\in B_k$.



For $k,k'\in\mathcal S_\phi$ and
$a\in\mathcal A_\phi(k)$, the abstract reward and transition kernel are
\begin{align*}
  r_\phi(k,a)
    &:=\sum_{s\in B_k}\omega_\phi(s\mid k)r(s,a),\\
  P_\phi(k'\mid k,a)
    &:=\sum_{s\in B_k}\omega_\phi(s\mid k)
       \sum_{s'\in B_{k'}}P(s'\mid s,a).
\end{align*}

Together with $\gamma$, these quantities define the abstract \MDP. We solve it
using the same lowest-index tie-breaking and lift its policy as $\widetilde\pi_\phi(s):=\pi_\phi(\phi(s))$.





Following \citet[Eq.~(1)]{fm2020}, we evaluate a fixed abstraction $\phi$ by
its maximum statewise value loss on the original \MDP,
$J(\phi):=\max_{s\in\mathcal S}
[V^*(s)-V^{\widetilde\pi_\phi}(s)]_+$, where
$[x]_+:=\max\{x,0\}$.
%
For a fixed $\phi$, $J(\phi)$ is the inner maximisation in their K-MDP gap
objective. Whereas their objective minimises over all admissible reduced state
spaces, Windowed compares only candidates induced within the declared divisor
window. We retain the worst-state criterion because an average under a chosen
initial-state distribution could conceal a large loss at an infrequently
weighted but decision-critical state.

\section{Windowed search}

Consider an MDP with two states that share the same optimal action, and have optimal values $V^*=(2,3)$. For a budget of $K=1$ abstract state, the value bin indices at $d=3/2,2,3$ are respectively
$(2,2),(1,2),(1,1)$. Thus, for
$F_K(d):=\mathds{1}\{N(d)\le K\}$, the two states are grouped together at $d=3/2$, separated at $d=2$, and grouped together again at $d=3$. Hence, $N(d) = 1,2,1$, respectively, showing that the condition $N(d) \leq K$ is not monotone in $d$. Thus, a binary trajectory can discard a feasible
interval; \cref{app:nonmono} gives a three-state missed-partition example.
Our conservative baseline, \RepairedBinary, retains the feasible endpoint
$d=V_{\max}$ before applying the midpoint updates of \citet{fm2020}. 

We address this problem by exploiting the structure of the abstraction in \cref{eq:family}. 
For \(V^*(s)>0\) and \(d>0\), the integer assignment \(\lceil V^*(s)/d\rceil\) remains constant except when \(d\) crosses a value \(V^*(s)/m\), where \(m\) is a positive integer. Because \(\pi^*(s)\) is fixed, the induced partition can change only at one of these divisor values. For a fixed closed window $W=[L,R]$ with $L>0$, we define \(\mathcal B(W)\) as the set containing the window endpoints and all values of \(d\) at which the induced partition may change:
\begin{equation}
  \mathcal B(W)
  :=
  \{L,R\}
  \cup
  \left\{
    \frac{V^*(s)}{m}:
    V^*(s)>0,\;
    m\in\mathbb N_+,\;
    L<\frac{V^*(s)}{m}<R
  \right\}.
  \label{eq:breakpoints}
\end{equation}
 These points are analytic, generally non-uniform, and not a numerical grid.
First-occurrence canonicalisation, denoted $\operatorname{can}(\phi)$, relabels
blocks as $0,1,\ldots$ in ground-state order so that label permutations of the
same partition are deduplicated.


\begin{theorem}[Closed-window completeness]
\label{thm:complete}
Assume $V^*(s)\ge0$ for every $s\in\mathcal S$, $0<L<R$, fixed tie-breaking
for $\pi^*$, and
$\{d\in[L,R]:N(d)\le K\}\ne\varnothing$.
Let $L=\beta_0<\dots<\beta_B=R$ be the sorted distinct points in
$\mathcal B(W)$. Then $\phi_d$ is constant on
$[\beta_i,\beta_{i+1})$ for every $i=0,\ldots,B-1$.
Consequently, the complete enumeration that evaluates one representative from
each interval, together with the endpoint $R$, realises every distinct
partition induced on $[L,R]$ and attains
$\min\{J(\phi_d):d\in[L,R],\;N(d)\le K\}$.
\end{theorem}

Let $d_b$ be the feasible divisor returned by \RepairedBinary. The reported
binary64 implementation uses
$d_{\min}:=\max\{10^{-10}V_{\max},10^{-12}\}$ and sets
$L:=\max\{d_b/\rho,d_{\min}\}$,
$R:=\min\{\rho d_b,V_{\max}\}$, and $W:=[L,R]$.
We fix $\rho=\windowRho$ before evaluation as a coverage--cost choice, not a
theoretically optimal value, and use $10^{-4}$ only as the stopping tolerance of
the Binary anchor search; neither is updated during a run. The numerical floor
is an implementation convention, not a theorem assumption. Every reported run
verifies $0<d_{\min}\le d_b\le V_{\max}$ and a nondegenerate window $L<R$.
The window is fixed, not updated: the guarantee covers $[L,R]$ but not
$d\notin[L,R]$. Because
$\operatorname{can}(\phi_{d_b})$ is retained, Windowed cannot return a larger
gap than \RepairedBinary under the same builder and evaluator. If $C_W$ is the
number of distinct feasible candidates in a completed window, then
$C_W\le|\mathcal S|\lceil V_{\max}/L\rceil+2$; the tighter incidence count and
binary64 guard are given in \cref{app:impl}.

\textsc{EndpointRepairedDivisorSearch} denotes only the divisor-selection loop
underlying Algorithm~4 of \citet{fm2020}, augmented to retain the feasible upper
endpoint $d=V_{\max}$. Its probes construct $\phi_d$ only to test
$N(d)\le K$ and return the anchor $d_b$; they do not construct or solve an
abstract \MDP. \Cref{alg:windowed} presents the complete Windowed procedure,
including the subsequent partition enumeration and
abstract-MDP evaluations.

\begin{algorithm}[H]
\caption{Windowed A-$K$-MDP}
\label{alg:windowed}
\begin{algorithmic}[1]
\Require solved MDP $M$ with $V^*\ge0$; lowest-index tie-broken $\pi^*$;
state budget $K$;
$\rho>1$; Binary stopping tolerance $\delta_b>0$
\State $V_{\max}\gets\max_{s\in\mathcal S}V^*(s)$
\State \textbf{assert} $V_{\max}>0$ and $N(V_{\max})\le K$
\State $d_{\min}\gets\max\{10^{-10}V_{\max},10^{-12}\}$
\State $d_b\gets
  \Call{EndpointRepairedDivisorSearch}{V^*,\pi^*,K,\delta_b}$
\State \textbf{assert} $N(d_b)\le K$ and
       $0<d_{\min}\le d_b\le V_{\max}$
\State $L\gets\max\{d_b/\rho,d_{\min}\}$
\State $R\gets\min\{\rho d_b,V_{\max}\}$;
       $W\gets[L,R]$ \Comment{fixed throughout the run}
\State \textbf{assert} $L<R$
\State construct and sort
       $\mathcal B(W)=\{\beta_0,\ldots,\beta_B\}$ using
       \cref{eq:breakpoints,eq:m-range},
       where $L=\beta_0<\cdots<\beta_B=R$
\State $\mathcal D\gets
       \{(\beta_i+\beta_{i+1})/2:0\le i<B\}
       \cup\{R,d_b\}$
\State $\mathcal C\gets
       [\,\operatorname{can}(\phi_{d_b})\,]$
       \Comment{retain Binary}
\For{$d\in\mathcal D$}
  \State $\phi\gets\operatorname{can}(\phi_d)$
  \If{$N(\phi)\le K$ and $\phi\notin\mathcal C$}
    \State append $\phi$ to $\mathcal C$
  \EndIf
\EndFor
\For{$\phi\in\mathcal C$ in Critical order (\cref{app:equal-budget}),
       Binary first}
  \State build and solve the compact MDP and compute $J(\phi)$
  \If{$J(\phi)=0$}
    \State \Return $\phi$, \textsc{zero-gap-found}
  \EndIf
\EndFor
\State \Return
$\arg\min_{\phi\in\mathcal C}J(\phi)$,
\textsc{all-candidates-checked}
\end{algorithmic}
\end{algorithm}

The algorithm shows the complete branch. The implementation first computes a
finite breakpoint-incidence bound. If it exceeds $I_{\max}=100{,}000$, a
deterministic bounded stream evaluates at most
$E_{\max}=\windowEvaluationCap$ distinct candidates. Reaching $J=0$ still
certifies the objective lower bound; a positive capped run is labelled
\textsc{lowest-gap-found} and has no complete-window claim. Exact bounds,
binary64 guards, run outcomes, and the proof of \cref{thm:complete} appear in
\cref{app:impl,app:proof}. Throughout, $E$ denotes the realised number of
abstract-MDP solves and $C_W$ the complete family size. Only the bounded branch
enforces $E\le E_{\max}$; complete runs may have $E>E_{\max}$.

\section{Results}

We evaluate \devRows\ case--$K$ pairs from \devMDPs\ solved \MDPs, seven ecological models and three controls. 
All comparison uses the full action set $\mathcal A$, the same abstract-\MDP builder and evaluator, differing only in the candidate partitions. 
Windowed improves \allWin\ rows and ties \allTie\
because it retains the repaired Binary candidate; on the \informativeRows\
informative completed windows, it improves/ties \infWin/\infTie. Three bounded
\emph{Aedes} runs reach $J=0$, whereas $K=305$ stops at $E_{\max}=\windowEvaluationCap$ and is reported as \textsc{lowest-gap-found}. \Cref{tab:main} and the appendix provide the remaining accounting.

\begin{table}[ht]
  \caption{Representative raw-native comparisons using the full action set
  $\mathcal A$. $C_W$ is the complete candidate-family size and $E$ the
  realised number of abstract-\MDP solves. Dashes denote bounded \emph{Aedes}
  runs; rows with $E<C_W$ stopped at $J=0$. The \SONA row is the local
  $K=50$ result. Gaps use model-native reward units and are comparable only
  within rows.}
  \label{tab:main}
  \centering
  \begin{tabular}{lrrrrrr}
    \toprule
    Model & $|\mathcal S|$ & $K$ & $C_W$ & $E$ & Binary $J$ & Windowed $J$ \\
    \midrule
    \emph{Aedes} & 6,097 & 61  & -- & \aedesSixtyOneEvals & \aedesSixtyOneBinary & \textbf{0} \\
                  & 6,097 & 305 & -- & \aedesEvalCount & \aedesBinary & \textbf{\aedesWindowed} \\
    Gouldian finch & \gouldThirteenStates & 8 & \gouldEightCandidates & \gouldEightCandidates & \gouldEightBinary & \textbf{\gouldEightWindowed} \\
                   & \gouldThirteenStates & 13 & \gouldThirteenCandidates & \gouldThirteenEvals & \gouldThirteenBinary & \textbf{0} \\
    Fisheries & 1,001 & 13 & \fishThirteenCands & \fishThirteenCands & \fishThirteenBinary & \textbf{\fishThirteenWindowed} \\
    \SONA & 819 & 50 & \sonaFiftyCandidates & \sonaFiftyCandidates & \sonaFiftyBinary & \textbf{\sonaFiftyWindowed} \\
    Fire & 91 & 4 & \fireWindowCandidates & \fireZeroEvals & \fireBinary & \textbf{0} \\
    Forest & 1,000 & 4 & \forestWindowCandidates & \forestZeroEvals & \forestBinary & \textbf{0} \\
    Reserve & 2,187 & 30 & \reserveWindowCandidates & \reserveZeroEvals & \reserveBinary & \textbf{0} \\
    \bottomrule
  \end{tabular}
\end{table}

\Cref{tab:main} reports representative rows; the complete ledger is in the
artifact. At Gouldian $K=13$, Critical reaches $J=0$ after
$E=\gouldThirteenEvals$ of $C_W=\gouldThirteenCandidates$ candidate solves.
Across the \informativeRows\ informative completed windows, Windowed
improves/ties \infWin/\infTie. The separately frozen extension improves/ties
five/three of its eight nonzero rows; its protocol is in
\cref{app:evaluation}. Equal-budget and cross-$K$ \SONA analyses are reported in
\cref{app:equal-budget,app:sona-cross-k}.

\begin{reviewresponse}

\paragraph{A changed management decision.}
In the \gouldThirteenStates-state fully observable Gouldian-finch projection
(not the original partially observable problem) \citep{chades2012momdps},
Binary at $K=13$ disagrees with the tie-broken ground policy in
\gouldThirteenBinaryDisagreements states and has
$J=\gouldThirteenBinary$. Windowed uses the same 13 blocks, restores every
ground-policy action, and reaches $J=0$. At state \gouldWitnessState, Binary
selects DN whereas the ground and Windowed policies select FG; the divisor
search, not the compactness budget, introduced this change.

\begin{table}[ht]
  \caption{Gouldian-finch comparison at $K=\gouldThirteenK$ with full action set
  $\mathcal A$. Agreement uses the fixed tie-broken ground policy; at
  zero-based state \gouldWitnessState, DN/FG denote do nothing/improved
  fire-and-grazing management \citep{chades2012momdps}.}
  \label{tab:policy-case}
  \centering
  \begin{tabular}{lrrrr}
    \toprule
    Selected policy & Blocks & Agreement & $J$ & Action at $s=\gouldWitnessState$ \\
    \midrule
    Ground reference & --- & 100\% & 0 & \gouldWitnessOptimalAction \\
    Binary \AKMDP & \gouldThirteenK & \gouldThirteenBinaryAgreement & \gouldThirteenBinary & \gouldWitnessBinaryAction \\
    \WindowedAKMDP & \gouldThirteenK & \textbf{\gouldThirteenWindowedAgreement} & \textbf{0} & \textbf{\gouldWitnessOptimalAction} \\
    \bottomrule
  \end{tabular}
\end{table}

\end{reviewresponse}

\section{Conclusions}

MDP policies are often simplified in an ad-hoc manner to increase uptake by conservation managers, but compact representation doesn't necessarily need to come at a performance loss. Improving on \citep{fm2020}, we have shown that midpoint binary search can miss A-K-MDP partitions because \(N(d)\) is not monotone. Our \WindowedAKMDP provides a systematic search of the distinct partitions within a stated window, improving 25 of 33 cases and tying the remainder. For conservation managers, this reduces the risk that a policy is simplified at an unnecessary cost to decision performance. Although compactness alone does not establish interpretability, it provides a stronger basis for examining and implementing simplified policies.

\clearpage
\bibliographystyle{plainnat}
\bibliography{reference}


\clearpage

\section*{Limitations and scope}

The guarantee is limited to the declared A-$K$-MDP family, divisor window, and
tie-breaking rule; it is not global over arbitrary partitions or all $d>0$.
The repeated budgets are descriptive rather than independent replications, and
neither small $K$ nor low $J$ establishes human interpretability, ecological
validity, or ease of implementation.

\section*{Code and evidence availability}

The accompanying anonymised artifact records the protocols, model identities,
candidate and policy hashes, verification scripts, and rows underlying each
reported table and figure.

\appendix
\section{Non-monotone feasibility}
\label{app:nonmono}

\begin{proposition}
\label{prop:nonmono}
The predicate $\mathbb{1}\{N(d)\le K\}$ can be non-monotone in $d$ even for two
states sharing one action, with $K=1$ and nonnegative values.
\end{proposition}

\begin{proof}
Take $V^*=(2,3)$. The value bins are $(2,2)$ at $d=3/2$, $(1,2)$ at $d=2$, and
$(1,1)$ at $d=3$, so feasibility is true, false, then true as $d$ grows. A
single state gives $N(d)\equiv1$, while two states with different optimal
actions give $N(d)\equiv2$; both cases are monotone. Thus the two-state,
shared-action example is minimal.
\end{proof}

The failure can also change which states are grouped. With $V^*=(3,5,8)$, one
shared action, and $K=2$, every $d\in[5/2,3)$ induces
\[
  \{\{1,2\},\{3\}\},
\]
whereas $d=4$ induces
\[
  \{\{1\},\{2,3\}\}.
\]
Bisecting $[0,8]$ probes $4,2,3,7/2,\dots$ and converges to the upper feasible
region without evaluating the earlier feasible interval
(\cref{fig:nonmono}).

\begin{figure}[ht]
  \centering
  \includegraphics[width=0.84\linewidth]{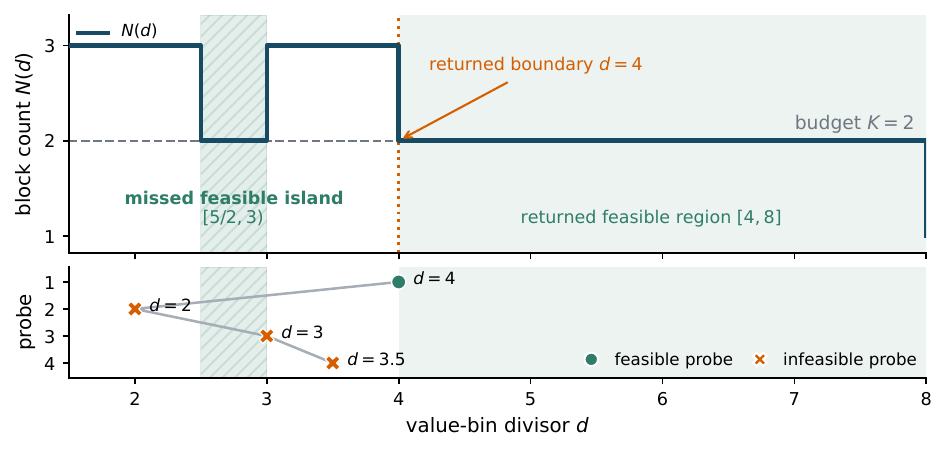}
  \caption{A binary-search trajectory that misses a feasible interval. The
  shaded regions satisfy $N(d)\le K=2$ for $V^*=(3,5,8)$. After probing
  $d=4,2,3,7/2$, the search contracts toward $4$ without evaluating
  $[5/2,3)$, whose induced partition differs from the returned partition.}
  \label{fig:nonmono}
\end{figure}

\section{Proof of \texorpdfstring{\cref{thm:complete}}{Theorem 1}}
\label{app:proof}

Fix $s$ with $v=V^*(s)>0$. For $m\ge2$,
$\lceil v/d\rceil=m$ exactly when $v/m\le d<v/(m-1)$; the value is one when
$d\ge v$. Thus $\lceil v/d\rceil$ changes only at $d=v/m$. At $b=v/m$ it
equals $m$, immediately left of $b$ it equals $m+1$, and immediately right of
$b$ it remains $m$. The coordinate is
therefore right-continuous and constant on each $[\beta_i,\beta_{i+1})$. A state
with $V^*(s)=0$ has $\lceil0/d\rceil=0$ for all $d>0$ and never changes, and
every action label $\pi^*(s)$ is fixed. Hence the full mapping vector, and so
the induced partition, is constant on each half-open interval; the mapping at
the closed endpoint $R$ is evaluated separately. Canonicalisation relabels blocks
by first occurrence, which removes only label permutations of the same
equivalence relation, and the deterministic builder, solver, and lift are
invariant to those. Exhaustive comparison over the resulting finite set therefore
attains the stated minimum. \qed

\section{Implementation notes}
\label{app:impl}

\paragraph{Finite candidate bound.}
For $v=V^*(s)>0$, only integers in
\begin{equation}
  \left\lfloor\frac{v}{R}\right\rfloor+1
  \le m\le
  \left\lceil\frac{v}{L}\right\rceil-1
  \label{eq:m-range}
\end{equation}
contribute breakpoints. Define
\begin{equation}
  I_W:=
  \sum_{v\in\operatorname{uniq}\{V^*(s):V^*(s)>0\}}
  \left|\left\{m\in\mathbb N_+:L<\frac{v}{m}<R\right\}\right|.
  \label{eq:breakpoint-incidences}
\end{equation}
The complete canonical candidate-family size $C_W$ satisfies
\begin{equation}
  C_W\le|\mathcal B(W)|\le I_W+2
  \le|\mathcal S|\left\lceil\frac{V_{\max}}{L}\right\rceil+2.
  \label{eq:candidate-bound}
\end{equation}
Only distinct feasible mappings require abstract-\MDP construction and
solution. A candidate evaluation is one such solve; $E$ is the realised number
and $E_{\max}$ is a cap, not a realised count.

\paragraph{Divisor search versus candidate evaluation.}
\textsc{EndpointRepairedDivisorSearch} is only the divisor-selection loop of
Algorithm~4 in \citet{fm2020}, augmented to retain $d=V_{\max}$. Its probes
construct $\phi_d$ only to test $N(d)\le K$; the abstract \MDP for the returned
anchor is built later when $\phi_{d_b}$ is evaluated as a candidate.

\paragraph{Floating-point implementation.}
\Cref{thm:complete} is an exact-arithmetic statement. In binary64 the midpoint
$(\beta_i+\beta_{i+1})/2$ can round to $\beta_{i+1}$ when consecutive breakpoints
are one unit in the last place (ULP) apart, so the representative would probe the
neighbouring interval. The implementation therefore evaluates each computed
boundary and its two adjacent representable values. Finite-precision
completeness is defined with respect to the distinct partitions generated under
the recorded binary64 convention. The saved canonical mappings and lifted
policies, rather than rounded decimal divisor values, identify the evaluated
candidates.

\paragraph{Bounded enumeration.}
Before enumerating, the implementation computes the breakpoint-incidence
count $I_W$ in \cref{eq:breakpoint-incidences}.
If $I_W\le I_{\max}=100{,}000$, it constructs the complete candidate set and
uses the deterministic ordering in \cref{app:equal-budget}. Otherwise, it
constructs a deterministic adaptive prefix by repeatedly probing the widest
remaining interval in $\log d$ and evaluates at most
$E_{\max}=\windowEvaluationCap$ candidates. The realised number of abstract-\MDP
solves is $E\le E_{\max}$. A capped run that ends with a positive value gap
reports \textsc{lowest-gap-found} and does not claim that every candidate in the
window was evaluated.

\paragraph{Meaning of the run outcomes.}
A run reports \textsc{all-candidates-checked} when every distinct feasible
partition induced within the declared window has been evaluated. A run may
stop earlier with \textsc{zero-gap-found} after verifying $J=0$, because zero
is the lowest possible value gap. A capped run that ends with $J>0$ reports
\textsc{lowest-gap-found}; this is the smallest gap among the evaluated
candidates but need not be the smallest gap over the complete window.

\section{Model scope and provenance}
\label{app:model-provenance}

\Cref{tab:model-provenance} states the exact computational object used in each
case. In particular, ``package instance'' is not synonymous with reproducing a
published table. Value gap is reported only within a row: different reward scales
and discounts make its magnitude incomparable across domains.

\begin{table}[ht]
  \caption{Artifact scope and discount factors for the ten-\MDP corpus. The
  downloaded Reserve and grey-wolf instances differ from the dimensions of the
  corresponding published cases; Gouldian is a fully observable latent-state
  projection rather than the original \MOMDP policy.}
  \label{tab:model-provenance}
  \centering
  \begin{tabularx}{\textwidth}{lrrr>{\raggedright\arraybackslash}X}
    \toprule
    Model & $|\mathcal S|$ & $|\mathcal A|$ & $\gamma$ & Artifact and treatment \\
    \midrule
    \emph{Aedes} & 6,097 & 17 & $1-10^{-8}$ & Checksum-locked official three-island generator \citep{peron2017simultaneous}; six-month step and near-undiscounted persistence objective. \\
    Fire & 91 & 2 & .96 & Constructed threatened-species fire-management replication. \\
    Fisheries & 1,001 & 11 & .96 & Downloaded \KMDP package instance. \\
    Forest & 1,000 & 2 & .99 & Constructed forest-management benchmark. \\
    Gouldian finch & 162 & 4 & .90 & Latent-state \MDP projection of the Gouldian \MOMDP \citep{chades2012momdps}; observations and initial belief omitted. \\
    Grey wolves & 1,000 & 4 & .96 & Downloaded package instance; not the dimensionally different published wolf cases. \\
    HIV & 6 & 2 & .97 & Deposited scenario 000 from a 72-model ensemble \citep{ahluwalia2021policy}. \\
    Maintenance & 10 & 4 & .97 & Seeded realization of the deposited generator \citep{ahluwalia2021policy}; rewards shifted by $+20$ per step, preserving value gap rankings but changing the $\phi_d$ family. \\
    Reserve & 2,187 & 7 & .96 & Downloaded package instance; distinct from the published 729-state, six-action case. \\
    \SONA & 819 & 4 & .96 & Downloaded sea-otter--northern-abalone package instance; state-coordinate file absent. \\
    \bottomrule
  \end{tabularx}
\end{table}

The \heldoutRows-row extension uses a \queueStates-state tandem-queue
controller from Marmote \citep{marmote2026}, a \inventoryStates-state
inventory-control model from QuantEcon \citep{quanteconDP}, and a
\busStates-state bus-engine replacement adapter derived from ruspy
\citep{ruspySoftware}. The accompanying artifact records the source URLs,
revisions, licences, generated arrays, and model hashes used in these
experiments.

\section{Evaluation details and representative results}
\label{app:evaluation}

All comparisons retain the full action set $\mathcal A$, respect
state-dependent availability, and use the same abstract-\MDP builder, solver,
lifting rule, and objective. The \heldoutRows-row extension uses tandem-queue,
inventory-control, and bus-replacement models. Before observing their gaps, we
set $K_0:=|\{\pi^*(s):s\in\mathcal S\}|$ and fixed
$K=K_0+\lceil q(|\mathcal S|-K_0)\rceil$ for
$q\in\{0,0.05,0.10,0.20\}$. Windowed improves/ties \RepairedBinary in
\heldoutWindowedWin/\heldoutWindowedTie\ of the \heldoutNonzero\ nonzero rows.

\section{Equal-budget search comparison}
\label{app:equal-budget}

Every comparator starts with Binary and targets
$E_{\mathrm{eff}}=\min\{\targetBudget,C_W\}$ distinct feasible mappings in the
same window. Critical explores endpoints and midpoints of feasible intervals
nearest $\log d_b$ first and removes duplicates. Over the \informativeRows-row
subset defined in Results, it outperforms/ties/underperforms the better of
log-grid and the \randomSeedCount-seed log-random median in \critVsBetter\ rows
(Fisheries 4/1/2, Gouldian 1/8/0, \SONA 1/0/0). Log-grid uses a base-2 van der
Corput sequence in $\log d$; log-random samples uniformly in log space for fixed
seeds 0--29. A random seed reaches the in-window optimum in every informative
row. Thus we claim complete in-window coverage, not universal search-order
superiority.

\section{Cross-budget \SONA analysis}
\label{app:sona-cross-k}

For the focused \SONA analysis, we sweep every integer
$K\in\{4,\ldots,50\}$ using the full action set $\mathcal A$, while respecting
state-dependent action availability, and combine the distinct partitions from
all 47 declared windows. Independently selected windows do not produce a
monotone at-most-$K$ result: the local search attains $J=0$ for
$K=6,\ldots,13$ but has positive value gap at $K=14$. Combining the distinct
candidates from all declared windows retains a six-block policy with $J=0$ and
100\% agreement with the tie-broken ground policy for every $K\ge6$. Thus six
blocks are sufficient within the declared collection of windows, but this does
not prove that six is minimal over all possible partitions. The plotted
normalised gap is
\[
  100\,J/\max\{\max_{s\in\mathcal S}|V^*(s)|,10^{-15}\};
\]
for this raw-native \SONA instance the denominator is
$\lVert V^*\rVert_\infty=10.666566655282647$.

\begin{figure}[ht!]
  \centering
  \includegraphics[width=0.90\linewidth]{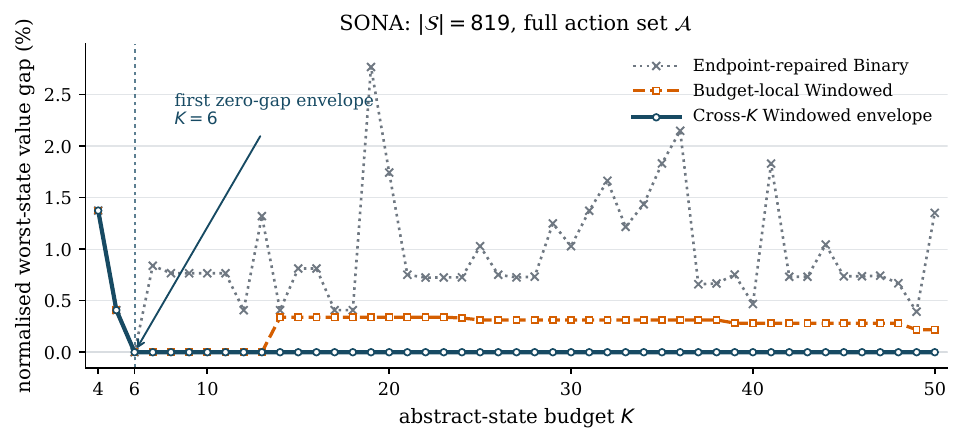}
  \caption{\SONA worst-state value gap as a percentage of
    $\lVert V^*\rVert_\infty$, using the full action set $\mathcal A$.
    Independently anchored local Windowed optima are non-monotone in $K$.
    Combining candidates across the declared windows first gives zero gap at
    $K=6$ and retains that policy for every larger budget.}
  \label{fig:sona-ks-frontier}
\end{figure}

For a fixed finite set of budgets $\mathcal K$, let $\mathcal C_k$ be the
complete candidate family from the declared window at $k$ and define
\begin{equation}
  \mathcal C^{\cup}:=\bigcup_{k\in\mathcal K}
  \{\operatorname{can}(\phi):\phi\in\mathcal C_k\},
  \qquad
  J^{\cup}(K):=
  \min_{\substack{\phi\in\mathcal C^{\cup}\\N(\phi)\le K}}J(\phi).
  \label{eq:cross-k}
\end{equation}
Exhausting every declared window gives the exact minimum over this fixed union,
and $J^{\cup}(K)$ is non-increasing because the eligible set can only grow with
$K$. This is not an optimum over arbitrary partitions.

\section{Optional action-set sensitivity}
\label{app:action-sensitivity}
This appendix experiment is separate from the core \AKMDP formulation.
At $K=50$, for each $B_A$ we exhaustively compare every retained action set
$\mathcal U\subseteq\mathcal A$ with $|\mathcal U|\le B_A$ and
$\mathcal U\cap\mathcal A(s)\ne\varnothing$ for every $s$. We solve each
restricted ground \MDP and use
$\mathcal A_{\phi,\mathcal U}(k):=
\mathcal U\cap\bigcap_{s\in B_k}\mathcal A(s)$ in its abstract \MDP, reporting
the lowest gap attained within the corresponding declared local windows.
I, AP, C, and H denote introduction, antipoaching, control, and half
AP/control. Windowed ties Binary at $B_A=2$ and reduces value gap by 27.0\% and
60.6\% at $B_A=3$ and $B_A=4$, respectively. Because the methods may select
different action sets, these results are not directly comparable with the
main-text experiment using the full action set $\mathcal A$.

\begin{table}[ht!]
  \caption{Optional \SONA action-set sensitivity at $K=50$, exhaustive over
    feasible retained action sets $\mathcal U\subseteq\mathcal A$ satisfying
    $|\mathcal U|\le B_A$ and the corresponding declared local Windowed
    families. This experiment is separate from the core state-budget
    comparison.}
  \label{tab:sona-action-budget}
  \centering
  \begin{tabular}{clrlr}
    \toprule
    $B_A$ & Binary actions & Binary gap
      & Windowed actions & Windowed gap \\
    \midrule
    2 & I+AP            & 0.382075 & I+AP     & 0.382075 \\
    3 & I+AP+C          & 0.059130 & I+AP+H   & \textbf{0.043192} \\
    4 & I+AP+C (3 used) & 0.059130 & I+AP+C+H & \textbf{0.023315} \\
    \bottomrule
  \end{tabular}
\end{table}

\end{document}